\documentclass[11pt,letterpaper]{article}

\usepackage[round, sort]{natbib}

\usepackage[margin=1in]{geometry}
\usepackage{xcolor}
\usepackage{booktabs} 
\usepackage{breakcites}
\usepackage{authblk}

\usepackage{amsmath,amsthm,amssymb}
\usepackage{subcaption}

\usepackage{tikz}
\usetikzlibrary{automata,positioning}
\usetikzlibrary{shapes,snakes}
\usetikzlibrary{decorations.pathmorphing}

\usepackage[shortlabels]{enumitem}
\usepackage[ruled]{algorithm2e}
\usepackage{algorithmic}

\usepackage[draft=false]{hyperref}
 \usepackage{hyperref}

\hypersetup{
     colorlinks = true,
     linkcolor = magenta,
     anchorcolor = blue,
     citecolor = teal,
     filecolor = blue,
     urlcolor = blue
     }
 
\newtheorem{theorem}{Theorem}[section]
\newtheorem{lemma}[theorem]{Lemma}
\newtheorem{proposition}[theorem]{Proposition}

\theoremstyle{definition}
\newtheorem{definition}[theorem]{Definition}

\DeclareMathOperator*{\argmin}{arg\,min}

\title{Provably Efficient Third-Person Imitation from Offline Observation}

\author[a]{Aaron Zweig}
\author[a, b, c]{Joan Bruna}
\affil[a]{Courant Institute of Mathematical Sciences, New York University, New York}
\affil[b]{Center for Data Science, New York University, New York}
\affil[c]{Institute for Advanced Study, Princeton}

\begin{document}

\maketitle

\begin{abstract}
	Domain adaptation in imitation learning represents an essential step towards improving generalizability.  However, even in the restricted setting of third-person imitation where transfer is between isomorphic Markov Decision Processes, there are no strong guarantees on the performance of transferred policies.  We present problem-dependent, statistical learning guarantees for third-person imitation from observation in an offline setting, and a lower bound on performance in the online setting.
\end{abstract}

\section{Introduction}

Imitation learning typically performs training and testing in the same environment.  This is by necessity as the Markov Decision Process (MDP) formalism defines a policy on a particular state space.  However, real world environments are rarely so cleanly defined and benign changes to the environment can induce a completely new state space.  Although deep imitation learning~\citep{ho2016generative} still defines a policy on unseen states, it remains extremely difficult to effectively generalize~\citep{duan2017one}.

Domain adaptation addresses how to generalize a policy defined in a source domain to perform the same task in a target domain~\citep{higgins2017darla}.  Unfortunately, this objective is inherently ill-defined.  One wouldn't expect to successfully transfer from a 2D gridworld to a self-driving car, but there is ambiguity in how to define a similarity measure on MDPs.

Third-person imitation~\citep{stadie2017third} resolves this ambiguity by considering transfer between isomorphic MDPs (formally defined in Section \ref{sec:prelim}), where the objective is to observe a policy in the source domain, and imitate that policy in the target domain.  In contrast to domain adaptation between unaligned distributions, the dynamics structure  constrains the space of possible isomorphisms, and in some cases the source and target may be related by a unique isomorphism.

We consider an idealized setting for third-person imitation with complete information about the source domain, where we perfectly understand the dynamics and the policy to be imitated. This work offers a theoretical analysis, in particular demonstrating that restricting to isomorphic MDPs with complete knowledge does not trivialize the problem. Specifically, regarding how the agent may observe the target domain, we consider two regimes, summarized in Figure \ref{fig:settings}:
\begin{itemize}
    \item In the \textbf{offline} regime (Section \ref{section:non-interactive}), an oracle perfectly transfers the source policy into the target domain, and the agent observes trajectories from the oracle policy (without seeing the oracle's actions). In this regime, we provide positive results establishing that with limited, state-only observations in the target domain, we can still efficiently imitate a policy defined in the source domain (Theorem \ref{thm:main}). 
    \item In the \textbf{online} regime (Section \ref{section:interactive}), the agent chooses policies in the target domain and draws trajectories. Our negative results in this setting (Theorem \ref{thm:negative}) prove that with full interaction in the target domain, imitation is extremely difficult in the presence of structural symmetry.
\end{itemize}

\paragraph{A Motivating Example:}
To clarify the setup and distinguish the two observation regimes, we elaborate upon an example.
Suppose our source domain is a video game, where the state space corresponds to the monitor screen and the action space corresponds to key presses.  And we wish to imitate an expert player of the game.
The target domain is the same game played on a new monitor with higher screen brightness.  Clearly the underlying game hasn't changed, and there is a natural bijection from screen states of the target monitor to those of the source monitor, namely ``dimming the screen".

On the one hand, in the offline setting, we're forbidden from playing on the new monitor ourselves.  Instead we observe recordings of the expert, played on the brighter monitor.  Again, as these are recordings, we see the states the expert visits but not their actions.
On the other hand, in the online setting, we simply run transitions on the brighter monitor.  Note that if the screen includes benign features which minimally impact the game (say the player's chosen name appears onscreen), it may be very difficult to learn the bijection between target and source monitor.
Either way, through observations we guess a new policy to played on the bright monitor, which hopefully mimics the expert's behavior.

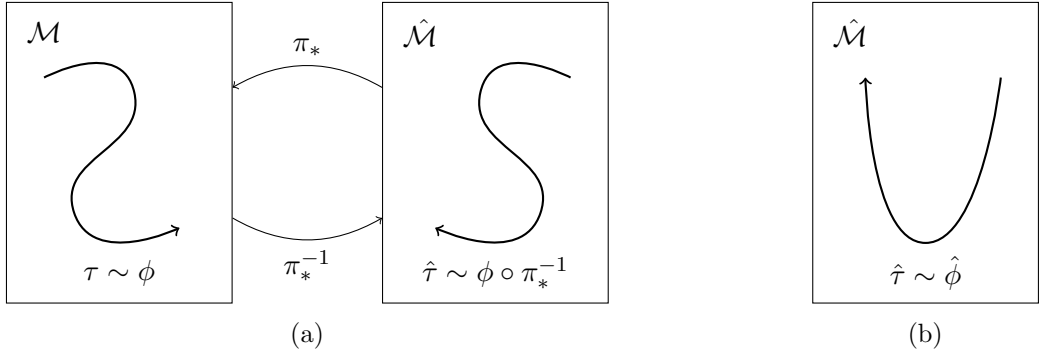
\begin{figure*}
    \centering
    \begin{subfigure}[b]{0.65\textwidth}
        \centering
        \scalebox{1}{
            \begin{tikzpicture}
                \node[label={[xshift=-1cm,yshift=1.6cm]center:$\mathcal{M}$}][draw, minimum width=3cm, minimum height=4cm](M) at (0,0) {};
                \node[label={[xshift=-1cm,yshift=1.6cm]center:$\hat{\mathcal{M}}$}][draw, minimum width=3cm, minimum height=4cm](M') at (5,0) {};
                
                \node[label={[yshift=-1.6cm]center:$\tau \sim \phi$}] (tau) at (0,0) {};
                \node[label={[yshift=-1.6cm]center:$\hat{\tau} \sim \phi \circ \pi_*^{-1}$}] (tau') at (5,0) {};

                \path[->] (M) edge[bend right, anchor=north] node {$\pi_*^{-1}$} (M');
                \path[->] (M') edge[bend right, anchor=south] node {$\pi_*$} (M);

                \begin{scope}
                \draw[black,->,thick] plot [smooth,tension=1.5]
                coordinates {(-1,1) (0.2,0.8) (-0.6, -0.8) (0.8, -1.0)};
                \end{scope}
                
                \begin{scope}
                \draw[black,->,thick] plot [smooth,tension=1.5]
                coordinates {(6,1) (4.8,0.8) (5.6, -0.8) (4.2, -1.0)};
                \end{scope}
            \end{tikzpicture}
        }
    \caption{}
    \label{fig:setting_offline}
    \end{subfigure}
    \begin{subfigure}[b]{0.33\textwidth}
    \centering
    \scalebox{1}{
        \begin{tikzpicture}
                \node[label={[xshift=-1cm,yshift=1.6cm]center:$\hat{\mathcal{M}}$}][draw, minimum width=3cm, minimum height=4cm](M') at (5,0) {};
                
                \node[label={[yshift=-1.6cm]center:$\hat{\tau} \sim \hat{\phi}$}] (tau') at (5,0) {};
                
                \begin{scope}
                \draw[black,->,thick] plot [smooth,tension=1.5]
                coordinates {(6,1) (5,-1.2) (4.2, 1.0)};
                \end{scope}
        \end{tikzpicture}
    }
    \caption{}
    \label{fig:setting_online}
    \end{subfigure}
    \caption{The observation regimes.  In the offline setting (a), the agent observes trajectories $\tau$ sampled from the policy $\phi$ that have been perfectly transferred into the isomorphic target domain.  In the online setting (b), the MDPs are still isomorphic but the agent only observes trajectories after playing their own policy $\hat{\phi}$.}
    \label{fig:settings}
\end{figure*}

\paragraph{Summary of Contributions:} Our primary contribution in this work is a provably efficient algorithm for offline third-person imitation, with an polynomial upper bound for the sample complexity necessary to control the imitation loss.  Our main technical novelty is a means of clipping the states of a Markov chain according to their stationary distribution, while preserving properties of a bijection between isomorphic chains.  We also prove an algorithm-agnostic lower bound for online third-person imitation, through reduction to bandit lower bounds.

\section{Setup}\label{sec:prelim}
\subsection{Preliminaries}
We consider a source MDP without reward $\mathcal{M} = \{S, A, P, p_0\} $, and target MDP $\hat{\mathcal{M}} = \{\hat{S}, A, \hat{P}, \hat{p}_0\}$.  
To characterize an isomorphism between $\mathcal{M}$ and $\hat{\mathcal{M}}$, we assume the existence of a bijective mapping $\pi_*: \hat{S} \rightarrow S$, such that $\hat{P}(s' | s, a) = P(\pi_*(s') | \pi_*(s), a)$ and $\hat{p}_0(s) = p_0(\pi_*(s))$.  Note that in this notation, $\pi_*$ is not a policy.

We also fix an ordering of the states $\hat{S}$ so that $\pi_*$ may be written in matrix form $\Pi_*$ as a permutation matrix.  In particular, we will overload notation to use $\pi_*$ as a permutation on $[|S|]$, such that $\pi_*(i) = j$ denotes that $\pi_*(\hat{s}_i) = s_j$.  Let $\mathcal{P}$ denote the space of $\hat{S} \rightarrow S$ permutation matrices.

A policy $\phi$ maps states to distributions on actions, but for our purposes it will be convenient to consider the policy as a matrix $\Phi: S \rightarrow S \times A$.  To relate the two notions, $\Phi$ is a block of diagonal matrices $\Phi_a: S \rightarrow S$ for each action, where $(\Phi_a)_{ii} = \phi(a | s_i)$, and $\Phi = [\Phi_{a_1} | \dots | \Phi_{a_{|A|}}]^T$.

The dynamics matrix is denoted $P: S \times A \rightarrow S$.  It can also be decomposed into blocks $P_a: S \rightarrow S$ where $(P_a)_{ij} = p(s_j | s_i, a)$, and $P = [P_{a_1}| \dots| P_{a_{|A|}}]$.

Using this notation, $\Phi^T P^T$ forms the Markov chain on $S$ induced by following policy $\phi$.  Explicitly,

\begin{equation}
    P_\phi(s' | s) = \sum_a \phi(a | s) P(s' | s, a) = \left(\Phi^T P^T \right)_{s, s'} 
\end{equation}

Note that under this notation, the dynamics and initial distribution in $\hat{\mathcal{M}}$ can be written as $\hat{P} = \Pi_*^T P (I \otimes \Pi_*)$ and $\hat{p}_0 = \Pi_*^T p_0$ respectively.
The occupancy measure $\rho_\phi$ is defined with regard to a policy, as well as the underlying dynamics and initial distribution.  Specifically, $\rho_\phi(s,a) = (1-\gamma) E_{s_0 \sim p_0, \tau \sim \Phi}\left[\sum_{i=0}^\infty \gamma^i \phi(a | s) P(s_i = s) \right]$, where the dependence on the dynamics $P$ is through the sampling of a trajectory $\tau$.

Similarly, we introduce the state-only occupancy measure $\mu_\phi(s) := \sum_a \rho_\phi(s,a)$.  We will make use of the identity $\rho_\phi(s,a) = \phi(a|s) \mu_\phi(s)$, as well as the fact that $\mu_\phi$ is the stationary distribution of the Markov chain $\Phi^T ((1-\gamma)p_0\mathbf{1}^T + \gamma P)^T$, which both follow from the constraint-based characterization of occupancy ~\citep{puterman1994markov}.

The value function for a given policy $\phi$ and reward function $R$ is defined as 
\begin{equation}
V_{\phi, R}(s) = E_{s_0 = s, \tau \sim \Phi}\left[\sum_{i=0}^\infty \gamma^i R(s_i, a_i) \right]~.    
\end{equation}
  We note the very useful identity $(1-\gamma)E_{s_0 \sim p_0}[V_{\phi, R}(s_0)] = \langle \rho_\phi, R \rangle$.

Lastly, we use the notation $\sigma_i(A)$ to denote the $i$th largest singular value of $A$.

\subsection{Observation Settings}

To begin, we're given full knowledge of the source domain $\mathcal{M}$, as well as $\Phi: S \times A \rightarrow S$ and $\rho_\Phi \in \mathbb{R}^{S \times A}$, the policy and corresponding occupancy measure we want to imitate.  We consider two settings through which we can interact with the target domain, in order to learn how to adapt $\Phi$ into this new domain.

\paragraph{Offline:} In the offline setting, we only observe the policy $\Phi_* := (I \otimes \Pi_*^T) \Phi \Pi_*$ being played in $\hat{\mathcal{M}}$.  We can consider $\Phi_*$ as an oracle for third-person imitation, as this policy exactly maps from $\hat{\mathcal{M}}$ to $\mathcal{M}$, calls $\Phi$, and maps back.  To guarantee the trajectories don't get trapped in a terminal state, we assume this agent has a $1-\gamma$ reset probability.  Through these observations, we must output a policy $\hat{\Phi}$ to be played in $\hat{\mathcal{M}}$.  We provide upper bounds for this setting in Section \ref{section:non-interactive}.

Crucially, in this setting we assume access to the states \textit{but not actions} from observed trajectories, in the imitation from observation setting~\citep{sun2019provably}.  This assumption is well-motivated.  In practice, observed trajectories from an expert often come from video, where actions are difficult to infer~\citep{liu2018imitation}.  Additionally, the problem becomes trivial with observed actions, as one may mimic the oracle's actions at each state in $\hat{S}$ without trying to understand $\Pi_*$ at all.

\paragraph{Online:} In the online setting, we define our own policy $\hat{\Phi}_t$ to play in $\hat{\mathcal{M}}$ at each timestep $t$, with full observation of the trajectories.  After $T$ total transitions we output our final policy $\hat{\Phi}$.  Intuitively, this setting allows for more varied observations in the target domain.  But without an expert oracle to demonstrate the correct state distribution, an agent in this setting may be deceived by near-symmetry in the dynamics and predict the wrong alignment.  We further highlight this difficulty in Section \ref{section:interactive}.

\subsection{Imitation Objective}

In either setting, through observations from the target domain we output a policy $\hat{\Phi}$.
The corresponding occupancy measure we denote as $\rho_{\hat{\Phi}, \Pi_*} \in \mathbb{R}^{\hat{S}\times A}$, where the subscript $\Pi_*$ reflects the dependence on the dynamics and initial distribution in $\hat{\mathcal{M}}$, namely $\Pi_*^T P (I \otimes \Pi_*)$ and $\Pi_*^T p_0$.

We measure imitation by comparing the correctly transferred policy $\Phi_*$ against the guessed policy $\hat{\Phi}$.  Explicitly, our objective is

\begin{equation}\label{objective}
    \inf_{\hat{\Phi}} g(\Phi, \Pi_*, \hat{\Phi}) := \inf_{\hat{\Phi}} TV\left((I \otimes \Pi_*^T)\rho_\Phi, \rho_{\hat{\Phi}, \Pi_*}\right)
\end{equation}

As a sanity check, we confirm that if we play $\hat{\Phi} = \Phi_* = (I \otimes \Pi_*^T) \Phi \Pi_*$, then indeed $\hat{\rho}_{\hat{\Phi}, \Pi_*} = (I \otimes \Pi_*^T) \rho_\Phi$ and the occupancies are equal.

A form of this objective with a general IPM as the distributional distance was introduced in~\citep{ho2016generative}.  To justify using this loss, note the objective can be equivalently written $\sup_{\|c\|_\infty \leq 1} E_{s \sim \Pi_*^T p_0}[V_{\phi_*, c}(s) - V_{\hat{\phi}, c}(s)] $.  In other words, minimizing imitation objective guarantees $\Phi_*$ and $\hat{\Phi}$ perform nearly as well for any reward function with a bound on maximum magnitude.


\section{Related Work}

The theory of imitation learning depends crucially on what interaction is available to the agent.  Behavior cloning~\citep{bain1995framework} learns a policy offline from supervised expert data.  With online data, imitation learning can be cast as a measure matching problem on occupancy measures~\citep{ho2016generative}.  With an expert oracle, imitation learning has no-regret guarantees~\citep{ross2011reduction}.  Numerous of these algorithms for imitation learning can be adapted to the observation setting~\citep{torabi2018behavioral,torabi2018generative,yang2019imitation}.

General domain adaptation for imitation learning has a rich applied literature~\citep{pastor2009learning,tobin2017domain,ammar2015unsupervised}.  Third-person imitation specifically was formalized in~\citet{stadie2017third}, extending the method of \citet{ho2016generative} by learning domain-agnostic features.  Other deep algorithms explicitly learn an alignment between the state spaces, based on multiple tasks in the same environments ~\citep{kim2019cross} or unsupervised image alignment~\citep{gamrian2019transfer}.

The closest work to ours is~\citet{sun2019provably}, which shares the focus on imitation learning without access to actions, but differs in studying the first-person setting primarily with online feedback.  This work also takes inspiration from literature on friendly graphs~\citep{aflalo2015convex}, which characterize robustly asymmetric structure.

\section{Offline Imitation}\label{section:non-interactive}

\subsection{Markov Chain Alignment}

Because the offline setting only runs policy $\Phi_*$, and reveals no actions, it is equivalent to observing a trajectory of the state-only Markov chain induced by $\Phi_*$ in $\hat{\mathcal{M}}$.  Let us elaborate on this fact.

Define the Markov chain $M := \Phi^T ((1-\gamma)p_0\mathbf{1}^T + \gamma P)^T$, which is ergodic when restricted to the strongly connected components that intersect the initial distribution.
In $\mathcal{\hat{M}}$, the dynamics are $\Pi_*^T P (I \otimes \Pi_*)$, the oracle policy is $(I \otimes \Pi_*)^T \Phi \Pi_*$, and the initial distribution is $\Pi_*^T p_0$.  We also assume the oracle agent following $\Phi_*$ has a $1-\gamma$ reset probability.

All together, this implies our observations in the offline setting are drawn from a trajectory of $\Pi_*^T M \Pi_*$.  In summary, given full knowledge of $M$ and a trajectory sampled from $\Pi_*^T M \Pi_*$, our algorithm will seek to learn the alignment $\Pi_*$ in order to approximate $\Phi_*$, hopefully leading to low imitation loss.

\subsection{Symmetry without approximation}

As a warmup, we consider the setting with no approximation where we observe $\Pi_*^TM\Pi_*$ exactly.  To relate this chain to $M$, we can try to find symmetries, i.e. the minimizers of

\begin{equation}\label{iso}
    \argmin_{\Pi \in \mathcal{P}} \|\Pi^T M \Pi - \Pi_*^T M \Pi_* \|_F~.
\end{equation}

We can equivalently consider finding automorphisms of $M$, which may be posed as a minimization over permutation matrices $\Pi: S \rightarrow S$:

\begin{equation}\label{auto}
    \argmin_\Pi \|\Pi^T M \Pi - M \|_F~.
\end{equation}

Clearly both these objectives are minimized at 0.  Intuitively, to recover $\Pi_*$ we'd like $\Pi_*$ to be the unique minimizer of (\ref{iso}), or equivalently $I$ to be the unique minimizer of (\ref{auto}).  Hence, in order to make third-person imitation tractable, we will seek to bound (\ref{auto}) away from $0$ when $\Pi \neq I$, or in other words focus on Markov chains which are robustly asymmetric.

We introduce notation:

\begin{definition}[Rescaled transition matrix]
     For an ergodic Markov chain $M$ with stationary distribution $\mu$, let $D = diag(\mu)$ and define $L = D^{1/2}MD^{-1/2}$ as the \emph{rescaled transition matrix} of $M$.
\end{definition}
\begin{definition}[Friendly matrix]
     A matrix $A$ is \emph{friendly} if, given the singular value decomposition $A = U\Sigma V^T$, $\Sigma$ has distinct diagonal elements and $V^T\mathbf{1}$ has all non-zero elements.  Similarly, a matrix $A$ is $(\alpha, \beta)$-friendly if $\sigma_\star := \min_i \sigma_i(A) - \sigma_{i+1}(A) > \alpha$ and $V^T\mathbf{1} > \beta \mathbf{1}$ elementwise.  An ergodic Markov chain $M$ is \emph{friendly} if its rescaled transition matrix $L$ is friendly.
\end{definition}

The significance of friendliness in graphs was studied in~\citet{aflalo2015convex}, to characterize relaxations of the graph isomorphism problem.  We first confirm several friendliness properties for Markov chains still hold.

\begin{proposition}\label{commuteprop}
    For a permutation matrix $\Pi$, $M = \Pi^T M \Pi$ if and only if $D = \Pi^T D \Pi$ and $L = \Pi^T L \Pi$.
\end{proposition}
\begin{proof}
    Suppose $M = \Pi^T M \Pi$.  If $\mu$ is the stationary distribution of $M$, then $(\mu^T \Pi) (\Pi^T M \Pi) = \mu^T \Pi$.  So by uniqueness of the stationary distribution in an ergodic chain, $\mu = \Pi^T \mu$ and therefore $D = \Pi^T D \Pi$.  Then clearly $D^{1/2} = \Pi^T D^{1/2} \Pi$ and therefore $L = \Pi^T L \Pi$.
    
    For the reverse implication, $\Pi^T M \Pi = \Pi^T D^{-1/2} L D^{1/2} \Pi = D^{-1/2} L D^{1/2} = M$.
\end{proof}

\begin{proposition}
    If $M$ is friendly, then it has a trivial automorphism group.
\end{proposition}
\begin{proof}

Suppose $M = \Pi^T M \Pi$.  Then by Proposition \ref{commuteprop}, $\Pi^TL^TL\Pi = L^TL = V\Sigma^2V^T$.  In particular, choosing $v$ as a column of $V$, $L^TLv = \sigma^2 v$ implies $L^TL\Pi v = \sigma^2 \Pi v$.  By friendliness, every eigenspace of $L^TL$ is one-dimensional, so $\Pi v = \pm v$.  And $\mathbf{1}^T \Pi v = \mathbf{1}^T v > 0$, so $\Pi v = v$ and therefore $\Pi = I$.
\end{proof}

In what follows, for any SVD, we will always choose to orient $V$ such that $V^T \mathbf{1} \geq 0$ elementwise.

\subsection{Exact Symmetry Algorithm}

By Proposition \ref{commuteprop}, the automorphism group of $M$ is contained in the automorphism group of the rescaled transition matrix $L$.  Interpreting $L$ as a weighted graph, determining its automorphisms is at least as computationally hard as the graph isomorphism problem~\citep{aflalo2015convex}.  

In general, algorithms for graph isomorphisms optimize time complexity, whereas we are more interested in controlling sample complexity.  Nevertheless, we have the following result:

\begin{theorem}\label{thm:exactsymmetry}
    Given $M$ and $\Pi_*^TM\Pi_*$, if $M$ is a friendly Markov chain, there is an algorithm to exactly recover $\Pi_*$ in $O(|S|^3)$ time.
\end{theorem}

This result is a simple extension of the main result in~\citet{umeyama1988eigendecomposition}, applying the friendliness property to Markov chains rather than adjacency matrices.  But the characterization of automorphisms will be used again later to control sample complexity, when we only observe $\Pi_*^TM\Pi_*$ through sampled trajectories.

We begin with the following:

\begin{proposition}\label{svdperm}
    Given two friendly matrices decomposed as $L_1 = U_1\Sigma V_1^T$ and $L_2 = U_2 \Sigma V_2^T$, suppose $L_2 = \Pi_*^T L_1 \Pi_*$.
        Then $\Pi_*$ is the unique permutation which satisfies $V_2 = \Pi^T V_1$.
\end{proposition}
\begin{proof}
    Clearly $L_2^TL_2 = \Pi_*^T L_1^TL_1 \Pi_*$.  Rewriting with the SVD gives $V_2 \Sigma^2 V_2^T = \Pi_*^T V_1 \Sigma^2 V_1^T \Pi_*$.

    Rearranging, this implies $V_2^T \Pi_*^T V_1$ commutes with $\Sigma^2$.  Commuting with a diagonal matrix with distinct elements implies $V_2^T \Pi_*^T V_1$ is diagonal.  As this product is also unitary and real, it must be that $V_2^T \Pi_*^T V_1 = S$ where $S$ is diagonal and $S^2 = I$.
    
    Again rearranging, this implies $\mathbf{1}^TV_1 = \mathbf{1}^T \Pi_*^T V_1 = \mathbf{1}^TV_2 S$.  By the assumption on the SVD orientation, $S$ must preserve signs, therefore $S = I$, and $V_2 = \Pi_*^T V_1$.
    
    Now, suppose $V_2 = \Pi^T V_1$.  Then $L_2^TL_2 = \Pi^T L_1^TL_1 \Pi$, so $\Pi_*^T\Pi$ is an automorphism of $L_2^TL_2$ and therefore $\Pi = \Pi_*$.
\end{proof}

\begin{proof}[Proof of Theorem \ref{thm:exactsymmetry}]
    Let $L_1$ and $L_2$ be the rescaled transition matrices of $M$ and $\Pi_*^TM\Pi_*$ respectively.  Reusing the same SVD notation, by Proposition \ref{commuteprop} and \ref{svdperm}, $V_2 = \Pi_*^T V_1$.  Consider the linear assignment problem $\min_{\Pi \in \mathcal{P}} \|V_2 - \Pi^TV_1\|_F$, which may be solved in $O(|S|^3)$ time using the Hungarian algorithm~\citep{kuhn1955hungarian}.  Again by Proposition \ref{svdperm}, this linear program is minimized at 0 and recovers $\Pi_*$ as the unique minimizer.
\end{proof}

\subsection{Symmetry with approximation}

With finite sample complexity, we still know the base chain $M$ exactly, but we get empirical estimates of the permuted chain $\Pi_*^TM\Pi_*$ by running trajectories.  Specifically, $m$ samples $(X_1, \dots, X_m)$ are drawn from $\Pi_*^TM\Pi_*$, with $X_1 \sim \Pi_*^Tp_0$.

Call the empirical estimate $\hat{M}$, i.e. $\hat{M}_{ij} = \frac{N_{ij}}{N_i}$ where $N_{ij}$ counts the number of observed $i \rightarrow j$ transitions and $N_i = \sum_{j} N_{ij}$.  And the empirical stationary distribution is $\hat{\mu}$ where $\hat{\mu}_i = \frac{N_i}{\sum_j N_j}$ and $\hat{D} = diag(\hat{\mu})$.
We can characterize the approximation error of the chain and stationary distribution as $E := \Pi_*\hat{M}\Pi_*^T - M$ and $\Delta = \Pi_*\hat{D}\Pi_*^T - D$ respectively.  Note these error terms are defined in the original state space $S$.

Our goal is to use $\hat{M}$ to produce a good policy in the target space.  Say we predict the bijection is $\Pi$, and play the policy $\hat{\Phi} = (I \otimes \Pi^T)\Phi \Pi$, whereas the correct policy in the target space is $\Phi_* = (I \otimes \Pi_*^T)\Phi \Pi_*$.  We'd like to be able to control the imitation distance between these two policies when $\Pi \approx \Pi_*$.

For that purpose, define $I_t(M) = \{i \in S: \mu_i \geq t$, $\mu^T = \mu^TM\}$, where $\mu$ is the stationary distribution of $M$, so these states will be visited ``sufficiently" often.  We first show correctness of the bijection on these states suffices for good imitation.

\begin{lemma}[Policy Difference Lemma~\citep{kakade2002approximately}]
    For two policies $\phi_1, \phi_2$ in the MDP defined by $\{S, A, P, R, p_0\}$,
    \begin{align*}
        E_{s \sim p_0}[V_{\phi_1,R}(s) - V_{\phi_2,R}(s)] . & = E_{\tau \sim \phi_1, p_0}\left[\sum_{t=0} \gamma^t A_{\phi_1, \phi_2}^R(s_t) \right] = \frac{1}{1-\gamma} \langle \mu_{\phi_1}, A_{\phi_1, \phi_2}^R \rangle~,
    \end{align*}
    where $A_{\phi_1,\phi_2}^R(s) = E_{a \sim \phi_1(\cdot | s)}[E_{s' \sim P}[R(s,a) + \gamma V_{\phi_2,R}(s') - V_{\phi_2,R}(s)]]$ is the average advantage function.
\end{lemma}

\begin{theorem}\label{ipmboundthm}
    Suppose $\pi^{-1}(s_i) = \pi_*^{-1}(s_i)$ for $i \in I_t(M)$.  Then $g(\Phi, \Pi_*, \hat{\Phi}) \leq  \frac{2t|S|}{(1-\gamma)^2}$.       
\end{theorem}
\begin{proof}

    First we decompose the objective
    
    \begin{align*}
    g(\Phi, \Pi_*, \hat{\Phi}) & = TV((I \otimes \Pi_*^T)\rho_\Phi, \rho_{\hat{\Phi}, \Pi_*})\\
    & = \sup_{\|c\|_\infty \leq 1} \langle \rho_{\Phi_*, \Pi_*} - \rho_{\hat{\Phi}, \Pi_*}, c \rangle\\
    & = \sup_{\|c\|_\infty \leq 1} E_{\hat{s} \sim \Pi_*^T p_0}[V_{\phi_*, c}(\hat{s}) - V_{\hat{\phi}, c}(\hat{s})]~.
    \end{align*}

    From the assumption and the definition of $\Phi_*$ and $\hat{\Phi}$, we have $\phi_*(\cdot | \hat{s}_i) = \hat{\phi}(\cdot | \hat{s}_i)$ whenever $i \in \pi_*^{-1}(I_t(M))$.  Equivalently, since $\mu_\phi$ is the stationary distribution of $M$ in the original space, and $\mu_{\phi_*} = \Pi_*^T \mu_{\phi}$, we have $\phi_*(\cdot | \hat{s}) = \hat{\phi}(\cdot | \hat{s})$ whenever $\mu_{\phi_*}(\hat{s}) \geq t$.
    
    Note that $\phi_*(\cdot | \hat{s}) = \hat{\phi}(\cdot | \hat{s})$ implies $A_{\phi_*, \hat{\phi}}^R(\hat{s}) = 0$ for any $R$.  Hence,
    
    \begin{align*}
    (1-\gamma) E_{\hat{s} \sim \Pi_*^T p_0}[V_{\phi_*,c}(\hat{s}) - V_{\hat{\phi}, c}(\hat{s})]
    & = \sum_{i \in \pi_*^{-1}(I_t(M))} \mu_{\phi_*}(\hat{s}_i)A_{\phi_*, \hat{\phi}}^c(\hat{s}_i) \quad + \sum_{i \notin \pi_*^{-1}(I_t(M))} \mu_{\phi_*}(\hat{s}_i)A_{\phi_*, \hat{\phi}}^c(\hat{s}_i)\\
    & \leq \sum_{i \notin \pi_*^{-1}(I_t(M))} t |A_{\phi_*, \hat{\phi}}^c(\hat{s}_i)|\\
    & \leq \frac{2t|S|}{1-\gamma}~,
    \end{align*}
    
    following from the simple bound $\max_s |A_{\phi_1, \phi_2}^c(s)| \leq \frac{2}{1-\gamma}$.
    
\end{proof}

The bound in Theorem $\ref{ipmboundthm}$ depends on $\Pi$ in a very discrete sense, controlled by the states where $\Pi$ and $\Pi_*$ agree.  Say $\Pi$ contains a single error, $\hat{s} = \pi^{-1}(s) = \pi_*^{-1}(s')$ for $s \neq s'$.  Then at $\hat{s}$ we mistakenly play the action distribution $\hat{\phi}(\cdot | \hat{s}) = \phi(\cdot | s)$, rather than the correct distribution $\phi_*(\cdot | \hat{s}) = \phi(\cdot | s')$.  Because we never observe actions from the oracle, $\phi$ could be arbitrarily different at $s$ and $s'$, yielding a very suboptimal occupancy measure.

\subsection{Approximate Symmetry Algorithm}



\begin{algorithm}[h]
\DontPrintSemicolon %
\KwIn{$P$, $\Phi$, $\gamma$, $p_0$, $t$, $(X_1, \dots, X_m)$}
\KwOut{A policy $\hat{\Phi}: \hat{S} \rightarrow \hat{S} \times A$}

$M \gets \Phi^T ((1-\gamma)p_0\mathbf{1}^T + \gamma P)^T$\;
$\mu \gets \textsc{Stationary}(M)$\;
\For{$(i,j) \in [|S|] \times [|S|]$}{
$N_{ij} \gets 0$\;
}
\For{$t \in [m-1]$}{
$N_{X_t, X_{t+1}} \gets N_{X_t, X_{t+1}} + 1$\;
}
$\hat{\mu} \gets 0$\;
$\hat{M} \gets 0$\;
\For{$i \in [|S|]$}{
 $\hat{\mu}_i \gets \sum_j N_{ij}/(m-1)$\;
\For{$j \in [|S|]$}{
 $\hat{M}_{ij} \gets N_{ij} / \sum_k N_{ik}$\;
}
}
 $D \gets \textsc{Diag}(\mu)$\;
 $\hat{D} \gets \textsc{Diag}(\hat{\mu})$\;
 $I_t \gets \{i \in [|S|]: \mu_i \geq t\}$\;
 $\hat{I}_t \gets \{i \in [|S|]: \hat{\mu}_i \geq t\}$\;
 $M \gets \textsc{Submatrix}(M, I_t, I_t)$\;
 $\hat{M} \gets \textsc{Submatrix}(M, \hat{I}_t, \hat{I}_t)$\;
 $D \gets \textsc{Submatrix}(D, I_t, I_t)$\;
 $\hat{D} \gets \textsc{Submatrix}(\hat{D}, \hat{I}_t, \hat{I}_t)$\;
 $U, \Sigma, V \gets \textsc{SVD}(D^{1/2}MD^{-1/2})$\;
 $\hat{U}, \hat{\Sigma}, \hat{V} \gets \textsc{SVD}(\hat{D}^{1/2}\hat{M}\hat{D}^{-1/2})$\;
 $\Pi' \gets \textsc{Hungarian}(V, \hat{V})$\;
 Choose any $\Pi \in \mathcal{P}$ such that $\forall i \in \hat{I}_t$, $\pi(i) = \pi'(i)$\;
\Return{$(I \otimes \Pi)^T \Phi \Pi$}\;

\caption{Permuted Policy Learning}\label{alg:ppl}
\end{algorithm}

In light of Theorem \ref{ipmboundthm}, an algorithm could either seek to recover $\Pi_*$ exactly, or find a $\Pi$ which agrees with $\Pi_*$ on high occupancy states.  We consider a learning algorithm for both objectives, and bound its sample complexity.  The trick will be carefully setting the threshold $t$ that defines what constitutes high occupancy.

To state the theorem, we introduce the subscript $t$ notation to denote the principle submatrix defined by the indices of $I_t$, and $g := \min_i|\mu_i - t|$ is the gap between the threshold and stationary values.  Lastly, we define:

\begin{definition}[Pseudospectral gap]
    The \emph{pseudospectral gap} of an ergodic Markov chain $M$ is $\gamma_{ps}(M) = \max_{k \geq 1} \frac{1-\lambda_2((D^{-1}M^TD)^kM^k)}{k}$, where $\lambda_2$ denotes the second largest eigenvalue.
\end{definition}

If $M$ is not ergodic, we will take $\gamma_{ps}(M)$ to mean the pseudospectral gap of $M$ restricted to the strongly connected components that intersect $p_0$.

\begin{theorem}\label{thm:main}
    The policy learning algorithm in Algorithm \ref{alg:ppl} satisfies the following: for $1 \geq \delta \geq 0$, $t > 0$, if $D_t^{1/2}M_tD_t^{-1/2}$ is $(\alpha, \beta)$-friendly and $m = poly\left(\frac{1}{\alpha}, \frac{1}{\beta}, \frac{1}{t}, |I_t|, \frac{1}{g}, \frac{1}{\gamma_{ps}(M)}, \log\frac{1}{1-\gamma}, \log |S|, \log \frac{1}{\delta} \right)$ then with probability at least $1-\delta$, the output policy $\hat{\Phi}$ satisfies $g(\Phi, \Pi_*, \hat{\Phi}) \leq \frac{2t|S|}{(1-\gamma)^2}$.  In particular, if $\min_i \mu_i > t$, $\hat{\Phi} = \Phi_*$.
\end{theorem}

The most important feature of this bound is the dependence on $|S|$.  In the sample complexity  it only appears through a log term, and all other terms can be independent of $|S|$ depending on the choice of $t$ and the structural properties of $M$.  The error is still linear in $|S|$, but this term appears necessary.  If some occupancy mass leaves the well-supported states $\pi_*^{-1}(I_t)$, it could cover all the negligible states, and either incur error linear in $|S|$, or require exploration of every state and therefore sample complexity linear in $|S|$.


\begin{proof}[Proof sketch]

Here we give the main ideas of the proof, full details are provided in the Appendix.

Remind that $\hat{M} = \Pi_*^T (M+E) \Pi_*$ and $\hat{D} = \Pi_*^T(D + \Delta)\Pi_*$.  We also define $\tilde{M} = M+E$ as the empirical chain permuted back into the original MDP.   Likewise define $\tilde{D} = D + \Delta$, and $\tilde{\mu}$ to be the diagonal of $\tilde{D}$.

Given $\hat{M}$ and $\hat{D}$, the immediate choice for an estimator of the rescaled transition matrix would be $\hat{D}^{1/2} \hat{M} \hat{D}^{-1/2}$.  However, this will not be well-defined if our samples don't visit every state of $\hat{S}$.
Furthermore, if $M$ is only ergodic when restricted to a subset of $S$, then $\hat{D}^{-1}$ won't be defined even with infinite sample complexity.  Similarly, if $\mu_* := \min_i \mu_i$ is vanishingly small, $m$ will become prohibitively large in order to guarantee that $\hat{D}^{-1}$ is well-defined.

Our primary technical novelty addresses both these issues by setting a threshold $t$ on stationary mass, and discarding states below the threshold.  Define $I_t = \{i \in [|S|] : \mu_i \geq t\}$ and $\hat{I}_t = \{i \in [|S|]: \hat{\mu}_i \geq t\}$.
We restate the notation that a subscript $t$ denotes taking the principle submatrix corresponding to $I_t$ or $\hat{I}_t$ depending on the matrix's domain.  So for example, $M_t$ is $M$ restricted to rows and columns given by $I_t$, and likewise $\hat{M}_t$ is $\hat{M}$ restricted to $\hat{I}_t$.

Several concentration results for empirical Markov chain transitions and stationary distributions control the convergence of our estimators~\citep{wolfer2019minimax,wolfer2019estimating}.  Our main assumption is that the gap $g = \min_i |\mu_i - t|$ is non-negligible.  Then with high probability and sample complexity depending on $g$ but not $\min_i \mu_i$, $\mu_i \geq t$ iff $\tilde{\mu}_i \geq t$.  In other words, no empirical stationary estimates will ''cross" the threshold, or put another way $\pi_*^{-1}(I_t) = \hat{I}_t$.  We can then restrict our attention to the states above the threshold, such that the sample complexity necessary for concentration $\tilde{M} \approx M$ depends on $t$ but not $\min_i \mu_i$ (and only logarithmically on $|S|$).

For $t > 0$, the restricted rescaled transition matrix $L_t = D_t^{1/2} M_t D_t^{-1/2}$ is well-defined.  And with high probability we can define our estimator $\hat{L}_t = \hat{D}_t^{1/2} \hat{M}_t \hat{D}_t^{-1/2}$.
Appealing to a strong friendliness assumption on $L_t$, singular value perturbation inequalities imply that $\hat{L}_t$ is also friendly.

Finally, the asymmetric properties of friendly matrices given in Proposition \ref{svdperm} enable exact recovery of the submatrix of $\Pi_*$ restricted to the indices $I_t$ and $\hat{I}_t$.  And by Theorem \ref{ipmboundthm}, determining the alignment on all high-occupancy states still yields a bound on the imitation loss.

\end{proof}

\section{Online Imitation}\label{section:interactive}

\subsection{MDP Alignment}

In the online setting, we're still seeking to imitate $\Phi$, or equivalently $\rho_\Phi$.  However, we no longer observe trajectories of the correct policy $(I \otimes \Pi_*)^T \Phi \Pi_*$ played in $\hat{\mathcal{M}}$.

Instead, we are in a setting similar to a bandit, but without reward.  At time $t$, we play a policy $\hat{\Phi}_t$ defined on $\hat{\mathcal{M}}$ and observe a transition.  We allow resets to the initial distribution.  After $T$ plays, where $T$ may be a random variable, we choose a final policy $\hat{\Phi}$ and receive instantaneous regret given by $g(\Phi, \Pi_*, \hat{\Phi})$.

One simple algorithm might treat each possible bijection as an arm, where pulling $\Pi$ is akin to running a trajectory using the policy $\Pi^T\Phi\Pi$, and then infer which alignment best matches the behavior policy.  Or one could consider algorithms which don't play policies of the form $\Pi^T\Phi\Pi$ but simply explore the target space in a principled way.

Nevertheless, we derive a lower bound on the imitation loss of any algorithm in the online setting, demonstrating even complete knowledge of the source domain doesn't trivialize third-person imitation.

\subsection{Lower Bound Counterexample}

Consider a small bandit-like MDP (Figure \ref{fig:source}).  Red corresponds to action $r$, blue corresponds to action $b$, and purple corresponds to both.  The numbers on the edges give transition probabilities when taking the associated action.  Let the initial distribution be $p_0(x_0) = p_0(y_0) = 1/2$.
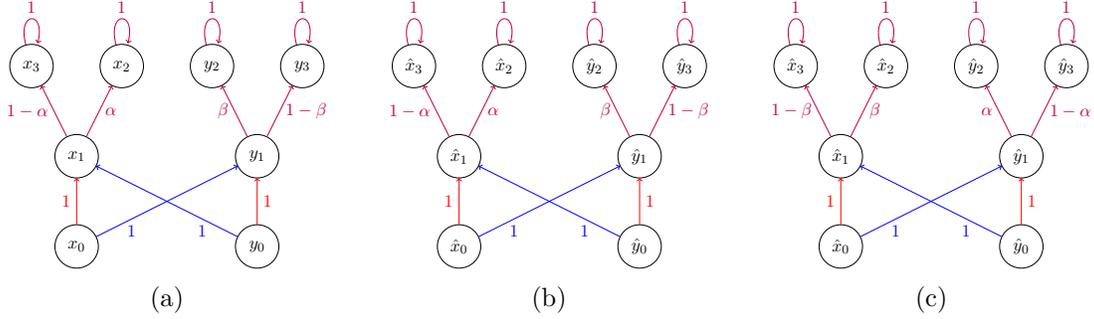
\begin{figure*}
    \centering
    \begin{subfigure}[b]{0.3\textwidth}
        \centering
        \scalebox{0.6}{
            \begin{tikzpicture}
                \node[state]             (x0) at (0,0) {$x_0$};
                \node[state]        (y0) at (4,0) {$y_0$};
                \node[state] (x1) at (0,2) {$x_1$};
                \node[state] (y1) at (4,2) {$y_1$};
                \node[state] (x2) at (1,4) {$x_2$};
                \node[state] (x3) at (-1,4) {$x_3$};
                \node[state] (y2) at (3,4) {$y_2$};
                \node[state] (y3) at (5,4) {$y_3$};
        
                \path[->]
                    (x0) edge[red,anchor=east] node {$1$} (x1)
                    (x0) edge[blue,pos=0.25,below] node {$1$} (y1)
                    (y0) edge[blue,pos=0.25,below] node {$1$} (x1)
                    (y0) edge[red,anchor=west] node {$1$} (y1)
                    (x1) edge[purple,anchor=west] node {$\alpha$} (x2)
                    (y1) edge[purple,anchor=east] node {$\beta$} (y2)
                    (x1) edge[purple,anchor=east] node {$1 - \alpha$} (x3)
                    (y1) edge[purple,anchor=west] node {$1 - \beta$} (y3)
                    (x2) edge[loop above, purple] node {$1$} (x2)
                    (x3) edge[loop above, purple] node {$1$} (x3)
                    (y2) edge[loop above, purple] node {$1$} (y2)
                    (y3) edge[loop above, purple] node {$1$} (y3);
            \end{tikzpicture}
        }
    \caption{}
    \label{fig:source}
    \end{subfigure}
    \begin{subfigure}[b]{0.3\textwidth}
    \centering
    \scalebox{0.6}{
        \begin{tikzpicture}
        \node[state]             (x0) at (0,0) {$\hat{x}_0$};
        \node[state]        (y0) at (4,0) {$\hat{y}_0$};
        \node[state] (x1) at (0,2) {$\hat{x}_1$};
        \node[state] (y1) at (4,2) {$\hat{y}_1$};
        \node[state] (x2) at (1,4) {$\hat{x}_2$};
        \node[state] (x3) at (-1,4) {$\hat{x}_3$};
        \node[state] (y2) at (3,4) {$\hat{y}_2$};
        \node[state] (y3) at (5,4) {$\hat{y}_3$};

        \path[->]
            (x0) edge[red,anchor=east] node {$1$} (x1)
            (x0) edge[blue,pos=0.25,below] node {$1$} (y1)
            (y0) edge[blue,pos=0.25,below] node {$1$} (x1)
            (y0) edge[red,anchor=west] node {$1$} (y1)
            (x1) edge[purple,anchor=west] node {$\alpha$} (x2)
            (y1) edge[purple,anchor=east] node {$\beta$} (y2)
            (x1) edge[purple,anchor=east] node {$1 - \alpha$} (x3)
            (y1) edge[purple,anchor=west] node {$1 - \beta$} (y3)
            (x2) edge[loop above, purple] node {$1$} (x2)
            (x3) edge[loop above, purple] node {$1$} (x3)
            (y2) edge[loop above, purple] node {$1$} (y2)
            (y3) edge[loop above, purple] node {$1$} (y3);
    \end{tikzpicture}
    }
    \caption{}
    \label{fig:target1}
    \end{subfigure}
    \begin{subfigure}[b]{0.3\textwidth}
    \centering
    \scalebox{0.6}{
        \begin{tikzpicture}
        \node[state]             (x0) at (0,0) {$\hat{x}_0$};
        \node[state]        (y0) at (4,0) {$\hat{y}_0$};
        \node[state] (x1) at (0,2) {$\hat{x}_1$};
        \node[state] (y1) at (4,2) {$\hat{y}_1$};
        \node[state] (x2) at (1,4) {$\hat{x}_2$};
        \node[state] (x3) at (-1,4) {$\hat{x}_3$};
        \node[state] (y2) at (3,4) {$\hat{y}_2$};
        \node[state] (y3) at (5,4) {$\hat{y}_3$};

        \path[->]
            (x0) edge[red,anchor=east] node {$1$} (x1)
            (x0) edge[blue,pos=0.25,below] node {$1$} (y1)
            (y0) edge[blue,pos=0.25,below] node {$1$} (x1)
            (y0) edge[red,anchor=west] node {$1$} (y1)
            (x1) edge[purple,anchor=west] node {$\beta$} (x2)
            (y1) edge[purple,anchor=east] node {$\alpha$} (y2)
            (x1) edge[purple,anchor=east] node {$1 - \beta$} (x3)
            (y1) edge[purple,anchor=west] node {$1 - \alpha$} (y3)
            (x2) edge[loop above, purple] node {$1$} (x2)
            (x3) edge[loop above, purple] node {$1$} (x3)
            (y2) edge[loop above, purple] node {$1$} (y2)
            (y3) edge[loop above, purple] node {$1$} (y3);
    \end{tikzpicture}
    }
    \caption{}
    \label{fig:target2}
    \end{subfigure}
    \caption{The bandit-like MDP, where (a) is the source domain, (b) is the target domain given $\Pi_1$ and (c) is the target domain given $\Pi_2$.}
    \label{fig:mdp}
\end{figure*}
In other words, the initial state is either $x_0$ or $y_0$.  Starting at $x_0$, the initial action is deterministic: playing $r$ leads to $x_1$, playing $b$ leads to $y_1$.  Starting at $y$ the actions lead to the opposite states.  Then the choice of action is irrelevant, and the transition to a terminal state is determined by $\alpha$ at $x_1$ and $\beta$ at $y_1$. 

This characterizes $\mathcal{M}$.  To introduce $\hat{\mathcal{M}}$, let's consider two possible bijections $\Pi_1$ and $\Pi_2$, which correspond to the possible target MDPs in Figure \ref{fig:target1} and Figure \ref{fig:target2} (note the values of $\alpha$ and $\beta$ are swapped given $\Pi_2$).
These correspond to two possible dynamics on our target space.  $\Pi_1$ is essentially the identity map, preserving states up to hats.  Whereas $\Pi_2(\hat{x}_i) = y_i$ and $\Pi_2(\hat{y}_i) = x_i$.

Finally, suppose the behavior policy we want to imitate in $\mathcal{M}$ is defined by $\phi(r | x_0) = 1$ and $\phi(b | y_0) = 1$.  In other words, the agent always travels in the first step to $x_1$.
That means, under $\Pi_1$ we want to travel to $\hat{x}_1$, and under $\Pi_2$ we want to travel to $\hat{y}_1$.
Intuitively, because $\rho_\Phi$ is highly asymmetric, but the MDP is nearly symmetric, one cannot choose a policy that performs well in multiple permutations of the MDP.  We formalize this intuition below.

\begin{theorem}\label{thm:negative}
Choose any positive values $\epsilon < \epsilon_0$ and $\delta < \delta_0$, where $\epsilon_0$ and $\delta_0$ are universal constants, and let $\alpha = 1/2 + \epsilon$ and $\beta = 1/2 - \epsilon$.  Consider any algorithm $\mathcal{A}$ that achieves $\gamma/4$-optimal imitation loss on the above MDP with probability at least $1-\delta$.  Then $E[T|\Pi_* = \Pi_i] = \Omega\left(\frac{1}{\epsilon^2}\log \frac{1}{\delta}\right)$ for some $i \in \{1,2\}$.
\end{theorem}

\begin{proof}

Fix a policy $\hat{\phi}$, and we will write $\rho_{\hat{\phi}, \Pi}$ as simply $\rho_{\Pi}$.

Again use the variational form of total variation to say $TV(\rho_1, \rho_2) = \sup_{\|c\|_\infty \leq 1} \langle \rho_1 - \rho_2, c \rangle$.  Choose $c$ so that $c(\hat{x}_i, a) = 1$ for $i \in \{1,2,3\}$ and $a \in A$, and 0 elsewhere.  Then a direct calculation gives $g(\Phi, \Pi_1, \hat{\Phi}) = TV((I \otimes \Pi_1^T)\rho_\Phi, \rho_{\Pi_1}) \geq \gamma - \gamma (\hat{\phi}(r|\hat{x}_0) + \hat{\phi}(b|\hat{y}_0))/2$.

Now we proceed by a reduction to multi-armed bandits with known biases.
Consider a two-armed bandit with Bernoulli rewards, where the hypotheses for arm biases are $H_1 = \{\alpha, \beta\}$ and $H_2 = \{\beta, \alpha\}$.
We define the following algorithm $\mathcal{B}$ for the two-armed bandit.  First run algorithm $\mathcal{A}$ on our MDP, where we couple pulls from arm 1 with transitions from $\hat{x}_1$ and pulls from arm 2 with transitions from $\hat{y}_1$.  Call $\hat{\phi}$ the policy output by $\mathcal{A}$.  Then output arm 1 if $\hat{\phi}(r|\hat{x}_0) > 1/2$, otherwise arm 2.

Under hypothesis $H_1$, $\Pi_* = \Pi_1$, so by our assumptions on $\mathcal{A}$, with probability at least $1-\delta$ we have $\gamma/4 > \gamma - \gamma (\hat{\phi}(r|\hat{x}_0) + \hat{\phi}(b|\hat{y}_0))/2$, which implies $\hat{\phi}(r|\hat{x}_0) > 1/2$.  Similar reasoning implies $\hat{\phi}(r|\hat{x}_0) \leq 1/2$ under $H_2$, hence $\mathcal{B}$ outputs the optimal arm with probability at least $1-\delta$.  Because $(\alpha, \beta) = (1/2 + \epsilon, 1/2-\epsilon)$, and the sample complexity of $\mathcal{A}$ is lower bounded by that of $\mathcal{B}$, the result then follows from Theorem 13 in~\citet{mannor2004sample}.

\end{proof}

This bound illustrates why imitation is substantially more challenging than seeking high reward.  In a regular RL problem with reward at the terminal states, if $\alpha \approx \beta$ then the expected reward changes very slightly depending on the policy.  But in the imitation setting, the value of $\alpha$ and $\beta$ are essentially features of the states, which the agent must (very inefficiently) distinguish in order to achieve error lower than $\gamma / 4$.  Likewise, this counterexample captures why the online setting is the more challenging one studied in this work.  In the offline regime, an oracle would only visit states on one half of the MDP and easily break the symmetry.

One may attribute this pessimistic bound to the choice of total variation distance.  Indeed, among IPMs, total variation has very poor generalization properties~\citep{sun2019provably}.  However, an alternative choice of IPM corresponds to a non-uniform prior over reward functions that the behavior policy is truly optimizing. If the prior strongly favored reward functions that smoothly depend on the local dynamics, then $c(\hat{x}_i, a) \approx c(\hat{y}_i, a)$ and this counterexample would no longer hold.  But this is a somewhat unnatural assumption, precluding for example a 2D gridworld with positive reward only at one state (since the gridworld would have many symmetries).

\section{Conclusion and Future Work}

In this paper, we introduced a theoretical analysis of third-person imitation, as an initial step in more fully understanding generalization in RL.  We demonstrated upper bounds for imitation learning across isomorphic domains under offline and state-only assumptions, and a lower bound for the online setting.  These bounds depend heavily on the structural properties of the dynamics and behavior policy, as well as the setting of third-person imitation where the domain adaptation is across isomorphic environments.

The upper bound dependence on structural and spectral properties is likely not optimal, although the dependence on $|S|$ in the error likely cannot be improved.  The lower bound is somewhat more robust, and any MDP with symmetry such that this bandit-like MDP can be embedded will suffer a similar lower bound on sample complexity.

The isomorphism assumption is certainly too strict in general.  However, weakening the assumption requires a characterization of MDP similarity, in order to decide when one should expect policy transfer through imitation to be feasible.  MDPs with features~\citep{krishnamurthy2016pac} could better characterize similarity, where the spectral features studied in this work could be combined with observed state features for more effective alignment through linear assignment.  Future work may include studying third-person imitation in the online setting for upper bounds, or exploiting MDP asymmetry in deep imitation.

\subsubsection*{Acknowledgements}

We are extremely grateful to David Brandfonbrener, Min Jae Song, and Raghav Singhal, who gave feedback and insightful suggestions throughout the work.

This work partially supported by the Alfred P. Sloan Foundation, NSF RI-1816753, NSF CAREER CIF 1845360, NSF CHS-1901091, Samsung Electronics, and the Institute for Advanced Study.



\appendix
\clearpage
\onecolumn
\section{Proof of Theorem \ref{thm:main}}

We define $D_{p_0} = \sum_{i} (p_0)_i^2/\mu_i$, where we interpret $0/0 = 0$.  Note that $p_0$ is absolutely continuous with respect to $\mu$, so this term is well-defined.

We now state the necessary concentration results.  Note the theorems are slightly altered from their statements in the literature, but follow immediately from their original proofs.  The first has a better dependence on $|S|$ by considering $L_2$ norm rather than $L_1$, and with slightly loose sample complexity.  The second is exactly an intermediate statement made in the theorem's original proof.

\begin{theorem}[Theorem 1 in~\citet{wolfer2019minimax}]\label{thm:chainconcentrate}
If $m = O\!\left( \!\frac{1}{\gamma_{ps}\epsilon^2\mu_i} \log\!\left(\frac{|S|\sqrt{D_{p_0}}}{\delta} \right) \! \right)$ then with probability at least $1-\delta/2$,
$\|M(i,\cdot) - \tilde{M}(i,\cdot)\|_2 \leq \epsilon$.
\end{theorem}

\begin{theorem}[Theorem 5.1 in \citet{wolfer2019estimating}]\label{thm:stationaryconcentrate}
If $m = O\!\left( \!\frac{1}{\gamma_{ps}\epsilon^2\mu_i} \log \! \left(\frac{|S|\sqrt{D_{p_0}}}{\delta} \right) \! \right)$ then with probability at least $1 - \delta/2$, $|\mu_i-\tilde{\mu}_i| \leq \epsilon \mu_i$.
\end{theorem}

We observe a simple consequence of the definition of the rescaled transition matrix:

\begin{proposition}\label{prop:fill}
    For an ergodic Markov chain $M$ with rescaled transition matrix $L$, $\sigma_1(L) = 1$ and $\gamma_{ps}(M) \geq 1 - \sigma_2(L)^2$.
\end{proposition}
\begin{proof}
    Choosing $k=1$ in the definition of $\gamma_{ps}$ gives the product $D^{-1}M^TDM = D^{-1/2}L^TLD^{1/2}$.  The first term itself is a Markov chain called the multiplicative reversiblization~\citep{paulin2015concentration}.  Because the chain has maximum eigenvalue $1$ and the eigenvalues of $L^TL$ are the squares of the singular values, it follows $\sigma_1(L) = 1$ and $\gamma_{ps} \geq 1 - \sigma_2(L)^2$.
\end{proof}

We set the occupancy threshold via $t$, and consider properties of the empirical estimators:

\begin{lemma}\label{lemma:inequalities}
    Let $g := \min_i |\mu_i - t|$.  Assume $g > 0$, $t > 0$, and $1/4 > \epsilon > 0$. If $m = O\left( \max\left(\frac{1}{\epsilon^2t}, \frac{1}{g^2} \right)\frac{1}{\gamma_{ps}} \log\left(\frac{|S|\sqrt{D_{p_0}}}{\delta} \right) \right)$, then with probability at least $1-\delta$, we have the following:
    \begin{enumerate}[(a)]
        \item $\pi_*^{-1}(I_t) = \hat{I}_t$
        \item $\hat{D}_t^{-1}$ is well-defined
        \item $\|E_t\|_F \leq \epsilon \sqrt{|I_t|}$
        \item $\|f_+(\Delta_t)^{1/2}\|_F \leq \sqrt{\epsilon}$ where $f_+(\cdot)$ is the elementwise absolute value.
        \item $\|(D_t + \Delta_t)^{-1/2} - D_t^{-1/2}\|_F \leq 2\sqrt{\frac{\epsilon |I_t|}{t}}$
    \end{enumerate}
\end{lemma}
\begin{proof}
    Note that for all $i \in I_t$, $\mu_i \geq t$.  So choosing precision $\epsilon$ and confidence $\frac{\delta}{2|I_t|}$ in Theorem \ref{thm:chainconcentrate} and Theorem \ref{thm:stationaryconcentrate},  taking a union bound over all $i \in I_t$, and noting $|I_t| \leq |S|$, we have that when $m = O\left( \frac{1}{\gamma_{ps}\epsilon^2 t} \log\left(\frac{|S|\sqrt{D_{p_0}}}{\delta} \right) \right)$, with probability at least $1-\delta/2$, $\|M(i,\cdot) - \tilde{M}(i,\cdot)\|_2 \leq \epsilon$ and $|\mu_i-\tilde{\mu}_i| \leq \epsilon \mu_i$.
    
    Additionally, choosing precision $\frac{g}{2\mu_i}$ and confidence $\frac{\delta}{2|S|}$ in Theorem \ref{thm:stationaryconcentrate}, and taking a union bound over all $i \in [|S|]$, when $m = O\left( \frac{1}{\gamma_{ps}g^2} \log\left(\frac{|S|\sqrt{D_{p_0}}}{\delta} \right)\right)$, with probability at least $1-\delta/2$ we have $|\mu_i-\tilde{\mu}_i| \leq g/2$.
    
    By the second application of the concentration results, for all $i \in [|S|]$, $|\mu_i - \hat{\mu}_{\pi_*^{-1}(i)}| = |\mu_i - \tilde{\mu}_i| \leq g/2 < g$.  So from the definition of $g$ it's clear that $\mu_i \geq t$ iff $\hat{\mu}_{\pi_*^{-1}}(i) \geq t$.  Hence, $i \in I_t$ iff $\pi_*^{-1}(i) \in \hat{I}_t$.
    
    If $i \in I_t$, $\mu_i \geq t$.  Hence $\hat{\mu}_{\pi_*^{-1}(i)} \geq \mu_i - g/2 > \mu_i - g \geq t > 0$.  This means each diagonal element of $\hat{D}_t$ is positive, hence it's invertible.
    
    By part (a), if we define $\Pi_{t^*}$ to be the restriction of $\Pi_*$ to the indices $I_t$ and $\hat{I}_t$, then $\Pi_{t^*}$ is still a permutation matrix.  Furthermore, $E_t = (\Pi_*\hat{M}\Pi_*^T - M)_t = \Pi_{t^*}\hat{M}_t\Pi_{t^*}^T - M_t = \tilde{M}_t - M_t$.
    
    Then $\|E_t\|_F^2 = \sum_{i \in I_t} \|\tilde{M}_t(i, \cdot) - M_t(i,\cdot)\|_2^2 \leq \sum_{i \in I_t} \|\tilde{M}(i, \cdot) - M(i,\cdot)\|_2^2 \leq \epsilon^2 |I_t|$.
    
    Similarly, $\|f_+(\Delta_t)^{1/2}\|_F^2 = \sum_{i \in I_t} |\mu_i - \tilde{\mu}_i| \leq \sum_{i \in I_t} \epsilon \mu_i \leq \epsilon$.
    
    To derive the last inequality, note that $|\mu_i - \tilde{\mu}_i| \leq \epsilon \mu_i$ implies $(1-\epsilon)\mu_i \leq \tilde{\mu}_i \leq (1+\epsilon)\mu_i$.  Therefore
    
    \begin{align*}
        \|(D_t + \Delta_t)^{-1/2} - D_t^{-1/2}\|_F^2 & = \sum_{i \in I_t} \left(\frac{1}{\sqrt{\tilde{\mu}_i}} - \frac{1}{\sqrt{\mu_i}} \right)^2\\
        & = \sum_{i \in I_t} \frac{\tilde{\mu}_i + \mu_i - 2\sqrt{\tilde{\mu}_i \mu_i}}{\tilde{\mu}_i \mu_i} \\
        & \leq \sum_{i \in I_t} \frac{(1+\epsilon)\mu_i + \mu_i - 2\sqrt{(1-\epsilon)\mu_i \mu_i}}{(1-\epsilon)\mu_i \mu_i} \\
        & \leq \frac{|I_t|}{t} * \frac{2+\epsilon - 2\sqrt{1-\epsilon}}{1-\epsilon} \\
        & \leq \frac{|I_t|}{t} * \frac{3\epsilon}{1-\epsilon}\\
        & \leq \frac{4\epsilon|I_t|}{t}
    \end{align*}
    
    where the second last inequality uses $\sqrt{1-\epsilon} \geq 1 - \epsilon$ for $1 > \epsilon > 0$.
    
\end{proof}

\begin{lemma}\label{lemma:friendlybound}
    If $L_t$ is $(\alpha, \beta)$-friendly for sufficiently large $\alpha$ and $\beta$, the matrix $\hat{L}_t$ is friendly if it is well-defined.
\end{lemma}
\begin{proof}
    Observe that $\tilde{L}_t := \Pi_{t^*}\hat{L}_t\Pi_{t^*}^T = (D_t + \Delta_t)^{1/2}(M_t + E_t)(D_t + \Delta_t)^{-1/2}$, so it suffices to show this matrix is friendly.
    
    We need the following bounds, utilizing the inequality $\sqrt{a+b} - \sqrt{a} \leq \sqrt{|b|}$:
    
    \begin{align*}
        \|(D_t + \Delta_t)^{1/2} - D_t^{1/2}\|_F & \leq \|f_+(\Delta_t)^{1/2}\|_F \leq \sqrt{\epsilon} \\
        \|D_t^{-1/2}\|_F & \leq \sqrt{\frac{|I_t|}{t}}\\
        \|(D_t + \Delta_t)^{1/2}\|_F & \leq \|D_t^{1/2}\|_F + \|f_+(\Delta_t)^{1/2}\|_F \leq 1 + \sqrt{\epsilon}\\
        \|M_t\|_F & \leq \sqrt{|I_t|}\\
        \|M_t + E_t\|_F & \leq (1 + \epsilon)\sqrt{|I_t|}\\
    \end{align*}
    
    Decompose the perturbation of $L_t$ as
    
    \begin{align*}
        \tilde{L}_t - L_t & = (D_t + \Delta_t)^{1/2}(M_t+E_T)((D_t + \Delta_t)^{-1/2}-D_t^{-1/2}) \\
        & + (D_t + \Delta_t)^{1/2}(M_t + E_t - M_t)D_t^{-1/2}\\
        & + ((D_t + \Delta_t)^{1/2} - D_t^{1/2})M_tD_t^{-1/2}
    \end{align*}
    
    Then we apply the inequalities above, using the triangle inequality and submultiplicativity to obtain
    
    \begin{align*}
        \|\tilde{L}_t - L_t\|_F & \leq \frac{16 \sqrt{\epsilon} |I_t|}{\sqrt{t}} + \frac{2\epsilon |I_t|}{\sqrt{t}} + \frac{\sqrt{\epsilon} |I_t|}{\sqrt{t}}\\
        & \leq \frac{19 \sqrt{\epsilon} |I_t|}{\sqrt{t}}
    \end{align*}
    
    Now decompose $L_t = U\Sigma V^T$ and $\tilde{L}_t = \tilde{U} \tilde{\Sigma} \tilde{V}^T$.

    By the Wielandt-Hoffman inequality~\citep{hoffman2003variation}, $\sum_i (\sigma_i(\tilde{L}_t) - \sigma_i(L_t))^2 \leq \|\tilde{L}_t - L_t\|_F^2$.  Therefore, if $\alpha = \min_i \sigma_{i}(L_t) - \sigma_{i+1}(L_t) > 2\|\tilde{L}_t - L_t\|_F$, then $\sigma_{i}(\tilde{L}_t) - \sigma_{i+1}(\tilde{L}_t) > 0$.
    
    By the Cauchy interlacing theorem and Propostion \ref{prop:fill}, $\sigma_1(L_t) \leq \sigma_1(L) = 1$.  And $\min_i \sigma_{i}(L_t)^2 - \sigma_{i+1}(L_t)^2 \geq \alpha^2$.
    
    Therefore, we can apply the Davis-Kahn theorem~\citep{yu2015useful} to conclude $1 - |\tilde{v}_i^Tv_i| \leq \zeta$ where
    \begin{align*}
        \zeta := \left(\frac{2\left(2+ \frac{19 \sqrt{\epsilon} |I_t|}{\sqrt{t}}\right) \frac{19 \sqrt{\epsilon} |I_t|}{\sqrt{t}}}{\alpha^2} \right)
    \end{align*} 
    
    Orienting $\tilde{V}$ so that $\tilde{V}^T\mathbf{1} \geq 0$, it follows $|\tilde{v}_i^Tv_i| = \tilde{v}_i^Tv_i$.

If $\beta > \sqrt{2|I_t|\zeta}$, then the friendliness assumption implies $v_i^T\mathbf{1} > \sqrt{2|I_t|\zeta}$ and therefore
\begin{align*}
    \tilde{v}_i^T\mathbf{1} & \geq  v_i^T\mathbf{1} - |v_i^T\mathbf{1} - \tilde{v}_i^T\mathbf{1}|\\
    & > \sqrt{2|I_t|\zeta} - \|\mathbf{1}\|_2\|v_i - \tilde{v}_i\|_2 \\
    & = \sqrt{2|I_t|\zeta} - \sqrt{|I_t|}\sqrt{1 + 1 - 2\tilde{v}_i^Tv} \\
    & > \sqrt{2|I_t|\zeta} - \sqrt{2|I_t|\zeta}\\
    & > 0
\end{align*}
\end{proof}

Now we can recover $\Pi_{t^*}$, using the decomposition $\hat{L}_t = \hat{U}\tilde{\Sigma}\hat{V}^T$.

\begin{lemma}\label{lemma:uniqueperm}
    Under the same assumptions as Lemma \ref{lemma:friendlybound}, if $|I_t|\zeta < \frac{1}{2}$, the unique permutation matrix $\Pi$ such that $\|\Pi^T V - \hat{V}\|_F \leq \sqrt{2|I_t|\zeta}$ is $\Pi_{t^*}$.
\end{lemma}
\begin{proof}
    By Proposition \ref{svdperm}, and the friendliness of $\hat{L}_t$ and $\tilde{L}_t$, we have that $\Pi_{t^*}^T\tilde{V} = \hat{V}$.  It follows that
    
    \begin{align*}
        \|\Pi_{t^*}^TV - \hat{V}\|_F & \leq \|\Pi_{t^*}^TV - \Pi_{t^*}^T\tilde{V}\|_F + \|\Pi_{t^*}^T\tilde{V} - \hat{V}\|_F\\
        & = \|V - \tilde{V}\|_F
    \end{align*}
    
    And note that $\|V-\tilde{V}\|_F^2 = \sum_i \|v_i - \tilde{v}_i\|_2^2 = \sum_i 2 - 2v_i^T\tilde{v}_i \leq 2|I_t|\zeta$.
    
    Conversely, 
    
    \begin{align*}
    \|\Pi - \Pi_{t^*}\|_F & = \|\Pi^T\tilde{V} - \Pi_{t^*}^T\tilde{V}\|_F\\
    & \leq \|\Pi^T\tilde{V} - \Pi^TV\|_F + \|\Pi^TV - \hat{V}\|_F + \|\hat{V} - \Pi_{t^*}^T\tilde{V}\|_F\\
    & \leq \sqrt{2|I_t|\zeta} + \sqrt{2|I_t|\zeta}\\
    & < \sqrt{2}
    \end{align*}
    
    Because distinct permutation matrices differ in Frobenius norm by at least $\sqrt{2}$, this guarantees $\Pi = \Pi_{t^*}$.
\end{proof}

\begin{proof}[Proof of Theorem \ref{thm:main}]

For a given $t$, suppose $D_t^{1/2}M_tD_t^{-1/2}$ is $(\alpha, \beta)$-friendly.  Then we choose $\epsilon$ to satisfy the following:

\begin{enumerate}
    \item $\alpha > 2\|\tilde{L}_t - L_t\|_F$
    \item $\beta > \sqrt{2|I_t|\zeta}$
    \item $|I_t|\zeta < 1/2$
\end{enumerate}

We observe these are all satisfied at $\sqrt{\epsilon} = O\left(\frac{\alpha^2 \beta^2 \sqrt{t}}{|I_t|^2} \right)$

If $m = poly\left(\frac{1}{\alpha}, \frac{1}{\beta}, \frac{1}{t}, |I_t|, \frac{1}{g}, \frac{1}{\gamma_{ps}(M)}, \log D_{p_0}, \log |S|, \log \frac{1}{\delta} \right)$, then Lemma \ref{lemma:inequalities} and \ref{lemma:friendlybound} imply the estimator $\hat{L}_t$ is friendly with probability at least $1-\delta$.  So by Lemma \ref{lemma:uniqueperm}, we conclude the permutation $\Pi'$ recovered from the Hungarian algorithm in Algorithm \ref{alg:ppl} agrees with $\Pi_*$ on $I_t$ and $\hat{I}_t$.  Finally, Theorem \ref{ipmboundthm} bounds the imitation objective.

Lastly, we rewrite the sample complexity, using the fact that $D_{p_0} \leq \frac{1}{1-\gamma}$ from the definition of $\mu$.  We also note that from Proposition \ref{prop:fill}, in the exact recovery setting $\min_i \mu_i > t$, we may replace $\gamma_{ps}$ in the sample complexity with $1 - \sigma_2(L_t)^2$.

\end{proof}

\end{document}